\documentclass{article}

\usepackage{microtype}
\usepackage[table,dvipsnames]{xcolor}
\usepackage{colortbl}
\usepackage{graphicx}
\usepackage{subcaption}
\usepackage{booktabs} 
\usepackage{amsmath,amssymb} 
\usepackage{color}

\usepackage{booktabs}
\usepackage{multirow}

\newcommand{\butterfly}{\mathbf{B}}
\newcommand{\givens}{\mathbf{G}}
\newcommand{\real}{\mathbb{R}}

\usepackage{hyperref}

\usepackage[accepted]{icml2026}

\usepackage{amsmath}
\usepackage{amssymb}
\usepackage{mathtools}
\usepackage{amsthm}

\usepackage{pifont}
\usepackage{enumitem}
\let\oldding\ding
\renewcommand{\ding}[2][1]{\scalebox{#1}{\oldding{#2}}}

\usepackage[capitalize,noabbrev]{cleveref}

\theoremstyle{plain}
\newtheorem{theorem}{Theorem}[section]

\theoremstyle{definition}

\theoremstyle{remark}

\usepackage[textsize=tiny]{todonotes}

\newcommand{\syz}[1]{\textcolor{blue}{SYZ:}}

\icmltitlerunning{ButterflyQuant: Ultra-low-bit LLM Quantization through Learnable Orthogonal Butterfly Transforms}

\begin{document}

\twocolumn[
\icmltitle{ButterflyQuant: Ultra-low-bit LLM Quantization through Learnable Orthogonal Butterfly Transforms}



\icmlsetsymbol{equal}{*}

\begin{icmlauthorlist}
\icmlauthor{Bingxin Xu}{aaa}
\icmlauthor{Zhen Dong}{bbb}
\icmlauthor{Oussama Elachqar}{ccc}
\icmlauthor{Yuzhang Shang}{ddd}
\end{icmlauthorlist}

\icmlaffiliation{aaa}{USC}
\icmlaffiliation{bbb}{UCSB}
\icmlaffiliation{ccc}{Omni}
\icmlaffiliation{ddd}{UCF}

\icmlcorrespondingauthor{Yuzhang Shang}{yuzhang.shang@ucf.edu}

\icmlkeywords{Machine Learning, ICML}

\vskip 0.3in
]



\printAffiliationsAndNotice{}  


\begin{abstract}
Large language models require massive memory footprints, severely limiting deployment on consumer hardware. 
Quantization reduces memory through lower numerical precision, but extreme 2-bit quantization suffers from catastrophic performance loss due to outliers in activations.
Rotation-based methods such as QuIP and QuaRot apply orthogonal transforms to eliminate outliers before quantization, using computational invariance: $\mathbf{y} = \mathbf{Wx} = (\mathbf{WQ}^T)(\mathbf{Qx})$ for orthogonal $\mathbf{Q}$.
However, these methods use fixed transforms--Hadamard matrices achieving optimal worst-case coherence $\mu = 1/\sqrt{n}$--that cannot adapt to specific weight distributions.
We identify that different transformer layers exhibit distinct outlier patterns, motivating layer-adaptive rotations rather than one-size-fits-all approaches.
In this work, we propose ButterflyQuant, which replaces Hadamard rotations with learnable butterfly transforms parameterized by continuous Givens rotation angles.
Unlike Hadamard's discrete $\{+1, -1\}$ entries that are non-differentiable and thus prohibit gradient-based learning, butterfly transforms' continuous parameterization enables smooth optimization while guaranteeing orthogonality by construction.
This orthogonal constraint ensures theoretical guarantees in outlier suppression while achieving $O(n \log n)$ computational complexity with only $\frac{n \log n}{2}$ learnable parameters.
We further introduce a uniformity regularization on post-transformation activations to promote smoother distributions amenable to quantization.
Learning requires only 128 calibration samples and converges in minutes on a single GPU.
\end{abstract}
\section{Introduction}

Large language models (LLMs) have demonstrated remarkable capabilities, but their deployment remains severely constrained by memory requirements. LLaMA-70B requires 140GB in FP16 precision, exceeding the capacity of most GPUs and making consumer deployment infeasible \citep{touvron2023llama2,zhao2025benchmarking}. Recent research highlights deployment challenges including memory bandwidth bottlenecks, with inference serving as the dominant cost in production systems \citep{chen2025efficientqat,yuan2024llm}. Quantization---reducing numerical precision to 2-4 bits---offers a direct solution by compressing LLMs 4-8$\times$. However, extreme quantization suffers from catastrophic performance degradation due to outliers in activations that dominate the dynamic range, a primary obstacle to low-bit compression~\citep{wei2022outlier,sun2024massive}.

\begin{figure*}
\centering
\includegraphics[width=0.99\textwidth]{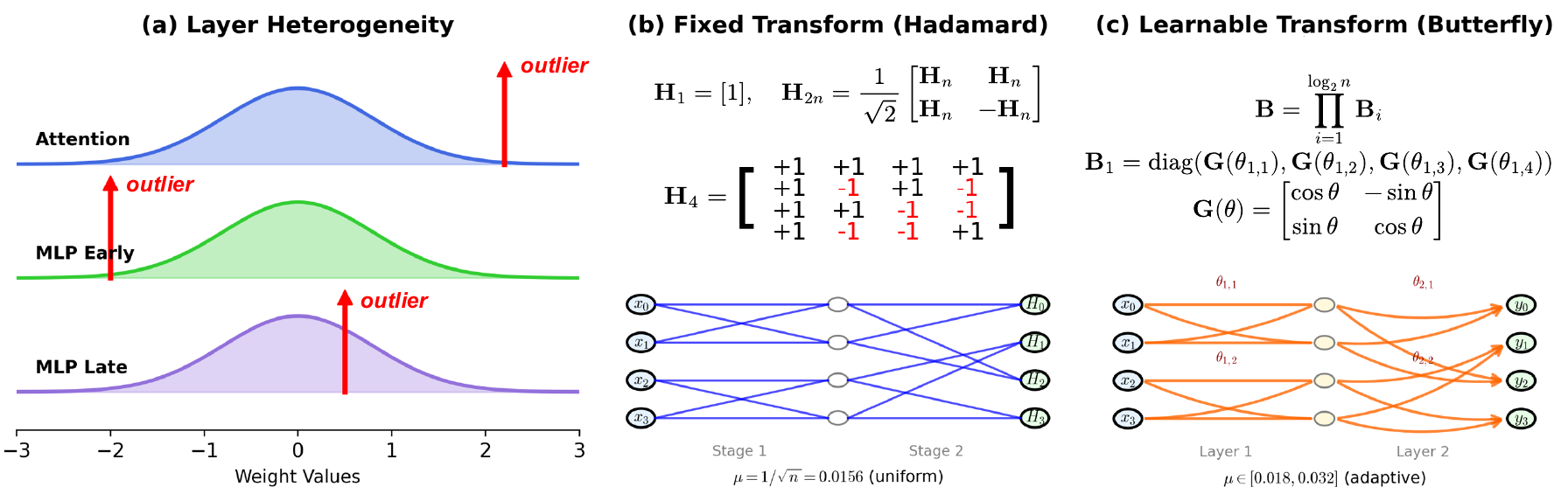}
\vspace{-0.1in}
\caption{\small \textbf{Layer heterogeneity motivates learnable transforms for LLM quantization.} \textbf{(a)} Different transformer layers exhibit distinct outlier distributions: attention (positive tails), early MLP (negative regions), late MLP (boundaries). \textbf{(b)} Hadamard transforms with discrete $\{+1,-1\}$ entries apply fixed rotations through recursive decomposition $\mathbf{H}_{2n} = \frac{1}{\sqrt{2}}[\mathbf{H}_n, \mathbf{H}_n; \mathbf{H}_n, -\mathbf{H}_n]$, achieving uniform coherence $\mu = 1/\sqrt{n} = 0.0156$ across all layers. \textbf{(c)} Butterfly transforms use continuous rotation angles $\theta_{i,j}$ in Givens rotations $\mathbf{G}(\theta)$, enabling gradient-based optimization to learn layer-specific patterns. This yields adaptive coherence that matches each layer's outlier distribution.}
\vspace{-0.2in}
\label{fig:teaser}
\end{figure*}

To mitigate the outlier problem, rotation-based quantization methods have emerged as a robust solution \citep{ashkboos2024quarot,chee2023quip}. These methods apply an orthogonal transformation $\mathbf{Q}$ before quantization, leveraging computational invariance: $\mathbf{y} = \mathbf{Wx} = (\mathbf{WQ}^T)(\mathbf{Qx})$. The rotation redistributes activations across channels, effectively smoothing out outlier features without altering the layer's output. Prominent methods like QuaRot \citep{ashkboos2024quarot} use fixed Hadamard transforms, which achieve optimal worst-case coherence, while QuIP \citep{chee2023quip} employs random orthogonal matrices. Despite their success, these approaches share a critical limitation: they apply a \emph{fixed}, data-agnostic transform with discrete $\{+1, -1\}$ entries that cannot be optimized via gradients to all layers.

This one-size-fits-all strategy is fundamentally misaligned with the nature of LLMs, which exhibit significant \emph{heterogeneity} across layers \citep{sun2024massive}. As illustrated in Figure~\ref{fig:teaser}(a), different transformer layers present unique quantization challenges: attention layers develop outliers in positive tails \citep{bondarenko2023quantizable}, early MLP layers show them in negative regions \citep{wei2022outlier}, and late MLP layers have them near distribution boundaries \citep{sun2024massive}. 
These distinct patterns arise from varied architectural roles. Attention's softmax operation naturally produces positive-skewed distributions \citep{bondarenko2023quantizable}, early MLP gating functions (e.g., SwiGLU \citep{shazeer2020glu}) create asymmetric negative activations, while deeper layers accumulate numerical artifacts at distribution extremes \citep{sun2024massive,dettmers2022gpt3}. This heterogeneity means that a single fixed rotation cannot be optimal for all layers. This layer-specific structure reveals a missed opportunity for optimization, motivating a shift from fixed to adaptive rotations.

To address this need for an adaptive yet efficient orthogonal transform, we propose \textbf{ButterflyQuant}. Our method replaces fixed Hadamard rotations with \emph{learnable} butterfly transforms--structured orthogonal matrices factorized into $O(n \log n)$ Givens rotations \citep{dao2019learning}. As shown in Figure~\ref{fig:teaser}(c), butterfly transforms maintain the efficient computational structure of Hadamard matrices but learn layer-specific rotation angles via gradient descent. Crucially, butterfly transforms parameterize rotations through continuous angles $\theta \in \mathbb{R}$, enabling smooth gradient flow and stable optimization---in contrast to Hadamard's discrete $\{+1, -1\}$ entries that prohibit gradient-based learning. This allows them to develop adaptive coherence patterns that are tailored to each layer's unique outlier distribution. 
Unlike other learnable methods (e.g., SpinQuant \citep{liu2024spinquant}) that optimize over the full Stiefel manifold with high computational cost, our sparse parameterization guarantees orthogonality by construction, enabling stable and efficient optimization. For non-power-of-2 dimensions common in LLMs (e.g., 5120), we develop composite transforms using Kronecker products. The learning process is remarkably lightweight, converging in minutes on a single GPU with a small calibration set.

Our contributions are: (1) We introduce learnable butterfly transforms for LLM quantization, maintaining orthogonality guarantees while adapting to model- and layer-specific weight distributions. (2) We prove that butterfly transforms can represent Hadamard matrices exactly, making them strictly more expressive than fixed schemes. (3) We achieve state-of-the-art 2-bit quantization results among rotation-based methods. (4) We show that the learned butterfly transform preserves the same $O(n \log n)$ butterfly computational graph as the fast Hadamard transform, enabling straightforward adaptation of existing optimized Hadamard CUDA kernels with expected minimal inference overhead.

\section{Related Work}
\subsection{Post-Training Quantization for LLMs}
Recent surveys \citep{zhao2025benchmarking} categorize LLM PTQ methods into four main approaches: compensation, rotation, salience, and optimization-based techniques.

\textbf{Compensation and Salience Methods.} GPTQ \citep{frantar2023gptq} uses second-order Hessian information for reconstruction error minimization. AWQ \citep{lin2023awq} preserves salient weights while quantizing others based on activation magnitudes. SmoothQuant \citep{xiao2023smoothquant} migrates difficulty from activations to weights via channel-wise scaling. OmniQuant \citep{shao2023omniquant} and AffineQuant \citep{ma2024affinequant} apply learnable transformations for outlier suppression. Mixed-precision methods like LLM.int8() \citep{dettmers2022gpt3} and SpQR \citep{dettmers2023spqr} maintain outlier channels in higher precision but increase complexity.

\textbf{Optimization Methods.} Recent approaches include VPTQ \citep{vptq2024} using vector quantization with optimized codebooks, APTQ \citep{aptq2024} with adaptive precision allocation, and AQLM \citep{egiazarian2024extreme} employing additive quantization for 2-bit compression. While achieving strong results, these methods require extensive optimization and lack theoretical guarantees of rotation-based approaches.

\subsection{Rotation-Based Quantization}

Rotation-based methods eliminate outliers through orthogonal transformations leveraging computational invariance: for orthogonal $\mathbf{Q}$, $\mathbf{y} = \mathbf{Wx} = (\mathbf{WQ}^T)(\mathbf{Qx})$. 

\textbf{Fixed Methods.} QuaRot \citep{ashkboos2024quarot} applies Hadamard transforms to redistribute outliers. QuIP \citep{chee2023quip} and QuIP\# \citep{quipsharp2024} use random orthogonal matrices based on the \emph{incoherence principle}, showing that maximizing incoherence minimizes worst-case quantization error. However, these predetermined rotations cannot adapt to specific models.

\textbf{Learned Methods.} SpinQuant \citep{liu2024spinquant} optimizes full $n \times n$ rotation matrices on the Stiefel manifold, requiring $O(n^2)$ parameters (16.7M for 4096-dim layers) and $O(n^3)$ operations. ROSAQ \citep{yoon2025rosaq} and KurTail \citep{akhondzadeh2025kurtail} learn saliency-aware and kurtosis-guided rotations respectively, but lack theoretical guarantees and still require $O(n^2)$ space. DuQuant~\cite{lin2024duquant} employs a greedy construction method that uses identified outlier dimensions as prior knowledge to build rotation matrices, followed by zigzag permutation to balance outlier distribution across blocks. 

\textbf{Our Position.} ButterflyQuant bridges fixed and learned approaches through structured parameterization. Unlike QuIP\#'s fixed rotations or SpinQuant's expensive Stiefel optimization, our Givens parameterization guarantees orthogonality by construction while enabling efficient data-driven learning, combining theoretical guarantees with adaptability. Crucially, the butterfly computational graph is structurally identical to the fast Hadamard transform, enabling adaptation of existing optimized inference kernels with minimal modifications.

\subsection{Structured Transforms and Orthogonal Parameterization}

Butterfly transforms \citep{cooley1965algorithm,dao2019learning} factorize dense matrices into $O(\log n)$ sparse layers with $O(n)$ parameters, successfully applied to attention \citep{dao2022flashattention}, state space models \citep{gu2021efficiently}, and neural architectures \citep{poli2023hyena,vahid2020butterfly}. Orthogonal parameterizations include Cayley transforms \citep{helfrich2018orthogonal} requiring matrix inversion, Householder reflections \citep{mhammedi2017efficient} with sequential dependencies, and Givens rotations \citep{givens1958computation,lezcano2019cheap} offering stable local updates. \citet{liu2024parameter} show butterfly-Givens achieve 10,000$\times$ parameter reduction in fine-tuning.

Despite extensive use in deep learning, structured transforms remain unexplored for quantization, which requires: (1) strict orthogonality for computational invariance $\mathbf{y} = \mathbf{Wx} = (\mathbf{WQ}^T)(\mathbf{Qx})$, (2) adaptability to layer-specific outliers \citep{sun2024massive}, and (3) efficient inference. While FlatQuant \citep{sun2024flatquant} uses Kronecker decomposition, it sacrifices orthogonality. ButterflyQuant uniquely combines continuous parameterization ($\theta \in \mathbb{R}$ enabling gradients vs. Hadamard's discrete $\{+1,-1\}$) and $O(n \log n)$ complexity, balancing expressiveness with efficiency. ButterflyQuant is the pioneer in leveraging these properties for LLM quantization.

\section{Method}

\subsection{Preliminaries: Rotation-Based Quantization and Incoherence}

Given a weight matrix $\mathbf{W} \in \real^{m \times n}$ and activation $\mathbf{x} \in \real^{n}$, standard quantization directly quantizes:
\begin{equation}
\mathbf{y} = \mathbf{W}\mathbf{x} \approx \text{Q}(\mathbf{W}) \cdot (\mathbf{x})
\end{equation}
where $\text{Q}(\cdot)$ denotes the quantization operation. This approach suffers from outlier features that dominate the dynamic range, causing severe accuracy degradation in extreme quantization settings.

\subsubsection{Theoretical Foundation: The Incoherence Principle}

The key insight from QuIP \citep{chee2023quip} is that quantization error is minimized when the rotation basis is maximally \emph{incoherent} with the standard basis. For an orthogonal matrix $\mathbf{Q} \in \mathbb{R}^{n \times n}$, the mutual coherence is defined as:
\begin{equation}
\mu(\mathbf{Q}) = \max_{i \neq j} |Q_{ij}|
\end{equation}

Lower coherence implies that information is more evenly distributed across all dimensions, preventing any single entry from dominating. QuIP proves that random orthogonal matrices achieve near-optimal incoherence with high probability, with $\mu(\mathbf{Q}) = O(\sqrt{\log n / n})$. This theoretical foundation explains why rotation-based methods outperform direct quantization: they transform weights and activations into a basis where values are more uniformly distributed.

\subsubsection{Fixed Hadamard Transforms}

QuaRot \citep{ashkboos2024quarot} applies Hadamard transforms as a computationally efficient approximation to random rotations. The Hadamard matrix $\mathbf{H}_n$ of dimension $n \times n$ is recursively defined as:
\begin{equation}
\mathbf{H}_1 = [1], \quad \mathbf{H}_{2n} = \frac{1}{\sqrt{2}} \begin{bmatrix} \mathbf{H}_n & \mathbf{H}_n \\ \mathbf{H}_n & -\mathbf{H}_n \end{bmatrix}
\end{equation}

For a weight matrix $\mathbf{W}$ and activation $\mathbf{x}$, QuaRot applies the transformation to obtain rotated versions:
\begin{equation}
\mathbf{W}' = \mathbf{W} \mathbf{H}^T, \quad \mathbf{x}' = \mathbf{H} \mathbf{x}
\end{equation}

This transformation leverages computational invariance---the property that the output remains unchanged while transforming both weights and activations:
\begin{equation}
\mathbf{y} = \mathbf{W}\mathbf{x} = (\mathbf{W}\mathbf{H}^T)(\mathbf{H}\mathbf{x}) = \mathbf{W}' \mathbf{x}'
\end{equation}

Hadamard matrices achieve coherence $\mu(\mathbf{H}_n) = 1/\sqrt{n}$, attaining the Welch bound---the theoretical minimum coherence achievable by any $n \times n$ orthogonal matrix. This theoretical elegance has made them the de facto choice for rotation-based quantization, and indeed, they deliver consistent improvements across diverse architectures. However, a fundamental limitation is that Hadamard matrices consist of discrete $\{\pm 1\}$ entries (with overall normalization $1/\sqrt{n}$), making them impossible to optimize via gradient descent. This discrete nature forces a one-size-fits-all approach that cannot adapt to the heterogeneous outlier patterns observed across transformer layers.

\subsubsection{From Fixed to Learnable Continuous Rotations}

Although Hadamard transforms achieve optimal worst-case incoherence, they suffer from two critical limitations: (1) Their discrete $\{+1, -1\}$ entries prohibit gradient-based optimization, and (2) neural networks exhibit layer-specific structured patterns \citep{sun2024massive}. As shown in Figure~\ref{fig:teaser}(a), attention layers, early MLPs, and late MLPs each have distinct outlier distributions that require tailored rotations. Fixed transforms cannot adapt to these heterogeneous patterns, treating all layers identically despite their vastly different quantization challenges.

We need rotations that maintain incoherence guarantees while adapting to specific layer patterns. \textit{Butterfly transforms overcome these limitations through continuous parameterization}: they use learnable angles $\theta \in \mathbb{R}$ that enable smooth gradient flow, allowing adaptation to each layer's specific outlier pattern while maintaining orthogonality guarantees. They factorize into $O(n \log n)$ Givens rotations, can represent Hadamard matrices exactly, and enable gradient-based optimization through their continuous parameterization.

\subsection{Butterfly Transforms: Bridging Fixed and Learnable Rotations}

We propose replacing fixed Hadamard rotations with learnable butterfly transforms to address the layer heterogeneity challenge. Butterfly transforms bridge rotation-based and optimization-based approaches: they maintain orthogonality guarantees while using continuous parameterization to learn layer-specific rotations that match the distinct outlier patterns of attention, early MLP, and late MLP layers. This adaptability is crucial for extreme quantization where different layers face fundamentally different challenges.

\subsubsection{Structure and Parameterization}

Let $\mathbf{B} \in \mathbb{R}^{n \times n}$ denote a butterfly transform matrix. It factorizes into $\log_2 n$ layers of sparse orthogonal matrices:
\begin{equation}
\mathbf{B} = \prod_{i=1}^{\log_2 n} \mathbf{B}_i
\end{equation}

Each layer $\mathbf{B}_i$ consists of $n/2$ independent $2 \times 2$ Givens rotations, where a Givens rotation is defined as:
\begin{equation}
\mathbf{G}(\theta) = \begin{bmatrix} \cos\theta & -\sin\theta \\ \sin\theta & \cos\theta \end{bmatrix}
\end{equation}

The layer structure is:
\begin{equation}
\mathbf{B}_i = \mathbf{P}_i \cdot \text{diag}(\mathbf{G}(\theta_{i,1}), \mathbf{G}(\theta_{i,2}), \ldots, \mathbf{G}(\theta_{i,n/2})) \cdot \mathbf{P}_i^T
\end{equation}

where $\mathbf{P}_i$ is a permutation matrix that defines the butterfly connectivity pattern at layer $i$, pairing indices with stride $2^{i-1}$. Specifically, layer 1 pairs adjacent indices (0,1), (2,3), ..., layer 2 pairs with stride 2: (0,2), (1,3), ..., and so on.

Unlike Hadamard's discrete entries, these continuous angles $\theta \in \mathbb{R}$ can be optimized through gradient descent, enabling the transform to adapt to layer-specific patterns identified during calibration. The butterfly structure creates a sparse, hierarchical factorization. For example, in 8 dimensions with 3 layers, each layer applies rotations to different index pairs:
\begin{equation}
\mathbf{B}_1 = \text{diag}(\mathbf{G}(\theta_{1,1}), \mathbf{G}(\theta_{1,2}), \mathbf{G}(\theta_{1,3}), \mathbf{G}(\theta_{1,4}))
\end{equation}
where the first layer pairs adjacent indices, the second layer pairs with stride 2, and the third with stride 4, creating the characteristic ``butterfly'' crossing pattern.

The complete transform achieves remarkable sparsity:
\begin{itemize}
\item \textbf{Parameters}: Only $\frac{n \log_2 n}{2}$ rotation angles (vs. $\frac{n(n-1)}{2}$ for full orthogonal)
\item \textbf{Complexity}: $O(n \log n)$ operations (vs. $O(n^2)$ for dense matrices)
\item \textbf{Sparsity}: $\frac{2\log_2 n}{n}$ non-zero ratio (e.g., 93.75\% sparse for $n=128$)
\end{itemize}

This parameterization ensures orthogonality by construction while enabling gradient-based optimization. See Appendix~\ref{app:butterfly_details} for detailed matrix structures and visualizations.

\subsubsection{Relationship to Hadamard Transforms}

\begin{theorem}
The Hadamard matrix $\mathbf{H}_n$ for $n = 2^k$ can be exactly represented as a butterfly transform with specific parameter choices \citep{dao2019learning}.
\end{theorem}

\begin{proof}[Proof Sketch]
The Hadamard matrix has a recursive structure that naturally maps to butterfly factorization. For $n = 2^k$, the Hadamard matrix can be factorized as:
\begin{equation}
\begin{aligned}
    \mathbf{H}_n &= \begin{bmatrix} \mathbf{H}_{n/2} & \mathbf{H}_{n/2} \\ \mathbf{H}_{n/2} & -\mathbf{H}_{n/2} \end{bmatrix} \\
    &= \begin{bmatrix} \mathbf{I}_{n/2} & \mathbf{I}_{n/2} \\ \mathbf{I}_{n/2} & -\mathbf{I}_{n/2} \end{bmatrix} \begin{bmatrix} \mathbf{H}_{n/2} & \mathbf{0} \\ \mathbf{0} & \mathbf{H}_{n/2} \end{bmatrix}
\end{aligned}
\end{equation}
This recursive decomposition continues until reaching $\mathbf{H}_2$. Each stage corresponds to a butterfly layer with specific parameters. The base case $\mathbf{H}_2$ can be expressed as:
\begin{equation}
\mathbf{H}_2 = \frac{1}{\sqrt{2}} \begin{bmatrix} 1 & 1 \\ 1 & -1 \end{bmatrix}
\end{equation}
which corresponds to a $2 \times 2$ orthogonal matrix with angle parameter $\theta = \pi/4$. Specifically, $\mathbf{H}_2 = \mathbf{G}(\pi/4) \cdot \text{diag}(1, -1)$, where the diagonal sign matrix is itself orthogonal and can be absorbed into the butterfly factorization. More generally, each $2 \times 2$ block in the Hadamard factorization is a product of a Givens rotation and a diagonal sign matrix, both of which are orthogonal and continuously parameterizable. The complete recursive factorization yields $\log_2 n$ sparse matrices, each implementable as a butterfly layer with specific angle and sign choices, establishing that Hadamard transforms are a special case of butterfly parameterizations.
\end{proof}

\subsubsection{Theoretical Coherence Analysis}

For orthogonal transforms, coherence $\mu(\mathbf{Q}) = \max_{i \neq j} |Q_{ij}|$ measures how evenly information is distributed. From compressed sensing \citep{candes2006robust,donoho2006compressed}, lower coherence reduces the sampling requirements for successful recovery:
\begin{equation}
m \geq C \cdot \mu^2(\mathbf{Q}) \cdot S \cdot \log n
\end{equation}
where $m$ is measurements, $S$ is sparsity, and $C$ is a constant.

\begin{figure}[t]
\centering
\includegraphics[width=0.48\textwidth]{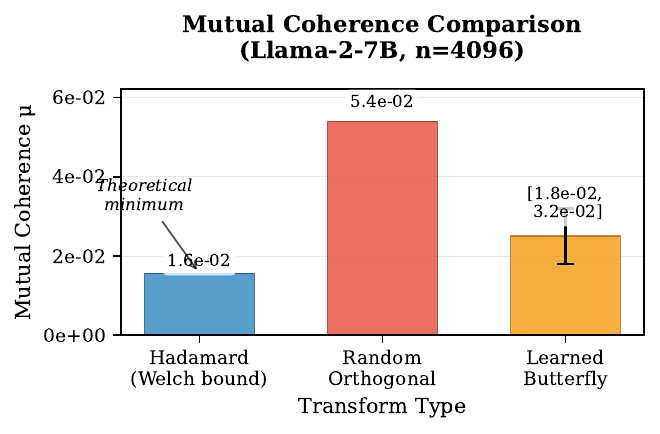}
\vspace{-0.2in}
\caption{Mutual coherence $\mu(\mathbf{Q})$ across transformer layers for different rotation strategies on LLaMA-2-7B. Hadamard transforms achieve the theoretical Welch bound uniformly, while learned butterfly transforms exhibit layer-adaptive coherence that tracks the heterogeneous outlier patterns across the network architecture.}
\vspace{-0.3in}
\label{fig:coherence}
\end{figure}

Figure~\ref{fig:coherence} compares coherence across LLaMA-2-7B layers ($n=4096$). Hadamard transforms achieve the Welch bound $\mu(\mathbf{H}_n) = 1/\sqrt{n} = 1.56 \times 10^{-2}$ uniformly. Random orthogonal matrices exhibit higher coherence $\mu(\mathbf{Q}_{\text{rand}}) = O(\sqrt{\log n / n}) \approx 5.4 \times 10^{-2}$ \citep{vershynin2018high}. Learned butterfly transforms demonstrate \emph{adaptive coherence} varying from $1.8$ to $3.2 \times 10^{-2}$ across layers, matching the heterogeneous outlier patterns: early attention layers maintain near-Welch-bound coherence for uniform decorrelation, while deeper MLP layers relax this constraint for their specific activation patterns.



\subsection{Composite Butterfly Transforms for Non-Power-of-2 Dimensions}

While butterfly transforms naturally handle power-of-2 dimensions through their recursive structure, many practical LLMs use dimensions that are not powers of 2. For instance, LLaMA-2-13B uses dimension 5120 = 40 × 128, where 40 is not a power of 2. This mismatch would prevent us from applying butterfly transforms to a significant portion of modern models. We address this challenge using composite transforms based on Kronecker products \citep{van1993approximation,de2000factor}, which allow us to combine smaller orthogonal transforms while maintaining computational efficiency.

\subsubsection{Kronecker Product Formulation}

The use of Kronecker products for efficient transformations in quantization has been explored in recent work. FlatQuant \citep{sun2024flatquant} employs Kronecker decomposition $\mathbf{P} = \mathbf{P}_1 \otimes \mathbf{P}_2$ to reduce the computational and memory overhead of their affine transformations, where $\mathbf{P}_1 \in \real^{n_1 \times n_1}$ and $\mathbf{P}_2 \in \real^{n_2 \times n_2}$ with $n = n_1 n_2$. Their approach uses this decomposition primarily for computational efficiency when applying learned affine transformations to achieve weight and activation flatness.

Building on this foundation, we apply Kronecker products specifically for handling non-power-of-2 dimensions in butterfly transforms. The critical distinction from FlatQuant lies in our constraint to orthogonal transformations: while FlatQuant employs general affine transforms $\mathbf{A}\mathbf{x} + \mathbf{b}$ with Kronecker-decomposed $\mathbf{A}$ that can distort norms and angles, our approach maintains strict orthogonality through $\mathbf{Q}_1 \otimes \mathbf{Q}_2$ where both $\mathbf{Q}_1$ and $\mathbf{Q}_2$ are orthogonal matrices. This orthogonal constraint provides theoretical guarantees---preserving inner products and norms---crucial for maintaining the computational invariance $\mathbf{y} = \mathbf{Wx} = (\mathbf{WQ}^T)(\mathbf{Qx})$ that underpins rotation-based quantization. Furthermore, we specifically employ butterfly transforms for power-of-2 components, achieving both the $O(n \log n)$ complexity of butterfly structures and the flexibility to handle arbitrary dimensions, while FlatQuant's general matrices require $O(n^2)$ parameters even with Kronecker decomposition.

For a dimension $d = d_1 \times d_2$, we construct the composite rotation as:
\begin{equation}
\mathbf{Q}_{\text{composite}} = \mathbf{Q}_1 \otimes \mathbf{Q}_2
\end{equation}
where $\mathbf{Q}_1 \in \real^{d_1 \times d_1}$ and $\mathbf{Q}_2 \in \real^{d_2 \times d_2}$ are orthogonal matrices.

The Kronecker product preserves orthogonality:
\begin{equation}
\begin{aligned}
    (\mathbf{Q}_1 \otimes \mathbf{Q}_2)^T (\mathbf{Q}_1 \otimes \mathbf{Q}_2) &= (\mathbf{Q}_1^T \mathbf{Q}_1) \otimes (\mathbf{Q}_2^T \mathbf{Q}_2) \\
    = \mathbf{I}_{d_1} \otimes \mathbf{I}_{d_2} &= \mathbf{I}_d
\end{aligned}
\end{equation}

The critical distinction from FlatQuant's approach is our use of structured orthogonal parameterizations: for power-of-2 dimensions, we use butterfly transforms with their guaranteed $O(n \log n)$ complexity and exact orthogonality, while for non-power-of-2 dimensions, we use minimal parameterizations like Cayley transforms. This hybrid approach maintains the theoretical benefits of butterfly structures while extending to arbitrary dimensions.

\subsubsection{Concrete Example for d = 5120}

For the 5120-dimensional hidden states, we use the factorization 5120 = 40 × 128. For $d_1 = 40$ (non-power-of-2), we parameterize $\mathbf{Q}_1$ using the Cayley parameterization, which maps skew-symmetric matrices to orthogonal matrices:
\begin{equation}
\mathbf{Q}_1 = (\mathbf{I} - \mathbf{A})(\mathbf{I} + \mathbf{A})^{-1}
\end{equation}
where $\mathbf{A}$ is skew-symmetric ($\mathbf{A}^T = -\mathbf{A}$) with zeros on the diagonal, requiring $\frac{40 \times (40-1)}{2} = 780$ parameters. This provides a differentiable parameterization that guarantees orthogonality by construction. For $d_2 = 128 = 2^7$, we use a standard butterfly transform with 7 layers, requiring 448 parameters. The total of 1,228 parameters achieves a 21,347× reduction versus a full 5120 × 5120 matrix.




\begin{table*}[t]
\centering
\caption{Comprehensive evaluation of 2-bit weight quantization (W2A16) on LLaMA-2 models. Some results are from the LLM quantization benchmarking paper \citep{zhao2025benchmarking}.}
\label{tab:main_results}
\scalebox{1.0}{
\begin{tabular}{l|cccccc|cc}
\toprule
\textbf{Method}  & \textbf{WinoG} & \textbf{PIQA} & \textbf{HellaS} & \textbf{ARC-e} & \textbf{ARC-c} & \textbf{MMLU} & \textbf{Wiki} & \textbf{C4}\\
\midrule
LLaMA2-7B (FP16) & 69.06 & 78.07 & 57.14 & 76.30 & 43.34 & 41.84 & 5.47 & 6.97 \\
\hline
GPTQ  & 48.93 & 57.13 & 28.15 & 32.11 & 20.22 & 22.97 & 36.77 & 79.06 \\
AWQ & 49.57 & 52.39  & 0.11 & 38.89 & 20.73 & 22.95 & 37.32 & 78.76 \\
OmniQuant  & 51.54 & 57.40 & 30.11 & 38.89 & 20.73 & 22.95 & 37.32  & 78.76 \\
QuIP & 51.07 & 59.25  & 30.11 & 38.89 & 20.73 & 22.95 & 37.32 & 78.76 \\
SpinQuant   & 55.40  & 65.15 & 46.89 &  58.56 & 29.20 & 24.72 & 16.43 & 23.10 \\
\rowcolor{brown!20} \textbf{ButterflyQuant}  & 62.27  & 68.97  & 48.43 & 62.58 & 29.86 & 26.68 & 15.40 & 16.61 \\\midrule
LLaMA2-13B (FP16) & 72.14 & 79.11 & 60.04 & 79.46 & 48.46 & 52.10 & 4.88 & 6.47 \\
\hline
GPTQ  & 52.09 & 62.24  & 34.80 & 42.59 & 21.25 & 23.00 & 20.05 & 19.10 \\
AWQ & 49.57 & 53.26 & 25.81 & 23.04 & 23.04 & 26.89 & 1.2e5 & 9.5e4 \\
OmniQuant  & 52.17   & 62.89   & 40.16 & 48.23 & 24.66 & 22.95 & 17.22 & 27.74 \\
QuIP & 55.72  & 65.45 & 39.65 & 51.56 & 25.85 & 23.79 & 13.75 & 14.71 \\
SpinQuant   & 58.22 & 56.96 & 44.34 & 51.28 & 25.57 & 26.85 & 17.10 & 19.85 \\
\rowcolor{brown!20} \textbf{ButterflyQuant}  & 62.91 & 69.28 & 50.10 & 62.64 & 30.49  & 29.83 & 10.24 & 12.48 \\
\bottomrule
\end{tabular}}
\vspace{-0.2in}
\end{table*}

\subsection{Loss Function}

We optimize butterfly parameters using a combination of reconstruction loss and uniformity regularization:
\begin{equation}
\mathcal{L} = \mathcal{L}_{\text{recon}} + \lambda_{\text{uniform}} \mathcal{L}_{\text{uniform}}
\end{equation}

The reconstruction loss minimizes the layer-wise output difference between original and quantized computations:
\begin{equation}
\begin{aligned}
    \mathcal{L}_{\text{recon}} &= \\
    \|\mathbf{W}\mathbf{x} - \text{Dequant}(\text{Quant}(\mathbf{W}\mathbf{B}^T)) &\cdot \text{Dequant}(\text{Quant}(\mathbf{B}\mathbf{x}))\|_2^2
\end{aligned}
\end{equation}
where $\text{Quant}(\cdot)$ denotes symmetric uniform quantization to $b$ bits:
\begin{equation}
\text{Quant}(x) = \text{clip}\left(\text{round}\left(\frac{x}{s}\right), -2^{b-1}, 2^{b-1}-1\right)
\end{equation}
and $\text{Dequant}(\cdot)$ is the corresponding dequantization: $\text{Dequant}(q) = s \cdot q$, where $s = \max(|x|) / (2^{b-1} - 1)$ is the scale factor computed per tensor or per group. This objective, similar to block-wise reconstruction in GPTQ \citep{frantar2023gptq} and learnable quantization methods \citep{shao2023omniquant}, directly minimizes the quantization-induced error while maintaining computational invariance through the butterfly transform.

\subsubsection{Uniformity Regularization}

We encourage uniform distribution across quantization bins by regularizing the rotated activations $\mathbf{x}' = \mathbf{B} \mathbf{x}$:
\begin{equation}
\mathcal{L}_{\text{uniform}} = D_{KL}(P_{\text{bins}}(\mathbf{x}') \| \mathcal{U})
\end{equation}
where $P_{\text{bins}}(\mathbf{x}')$ denotes the empirical distribution of quantized values across the $2^b$ quantization bins, and $\mathcal{U}$ is the uniform distribution over these bins.

While traditional quantization methods apply uniformity regularization to weight distributions \citep{park2017weighted,baskin2021uniq}, we specifically target activations for two key reasons: (1) Theoretical justification from information theory: uniform quantization achieves maximum entropy for a given bit-width, and applying this to activations ensures optimal information preservation through the quantized layer \citep{cover1999elements}; (2) The rotation-quantization duality $\mathbf{y} = \text{Q}(\mathbf{W}\mathbf{B}^T) \cdot \text{Q}(\mathbf{B}\mathbf{x})$ means that uniformizing activations through $\mathbf{B}$ simultaneously improves both weight and activation quantization.

\subsection{Mathematical Properties}

\begin{theorem}[Expressive Power of Butterfly Transforms]
Butterfly transforms with $\log_2 n$ layers can exactly represent important structured orthogonal matrices---including the Hadamard, Discrete Fourier Transform (DFT), and Discrete Cosine Transform (DCT) matrices---and form a universal building block for any structured matrix admitting an $O(n \log n)$ fast algorithm \citep{cooley1965algorithm,dao2022monarch}. While butterfly transforms with $\frac{n \log n}{2}$ parameters cannot represent arbitrary $n \times n$ orthogonal matrices (which require $\frac{n(n-1)}{2}$ degrees of freedom), they achieve favorable approximation-complexity tradeoffs for structured matrices encountered in practice. (Proof sketch in Section~\ref{app:butterfly_details})
\end{theorem}

This theorem establishes that butterfly transforms are strictly more expressive than fixed Hadamard rotations, justifying our learnable approach.

\subsubsection{Gradient Flow Through Butterfly Layers}

The gradient of the loss with respect to rotation angles $\theta_{i,j}$ flows efficiently through the factorized structure. For a single Givens rotation $\mathbf{G}(\theta)$, the gradient is:
\begin{equation}
\frac{\partial \mathbf{G}(\theta)}{\partial \theta} = \begin{bmatrix} -\sin\theta & -\cos\theta \\ \cos\theta & -\sin\theta \end{bmatrix}
\end{equation}
This smooth, bounded gradient ensures stable optimization without gradient explosion or vanishing, unlike discrete Hadamard transforms where gradients are undefined.

    

\begin{figure}[t]
    \centering
    \includegraphics[width=0.9\linewidth]{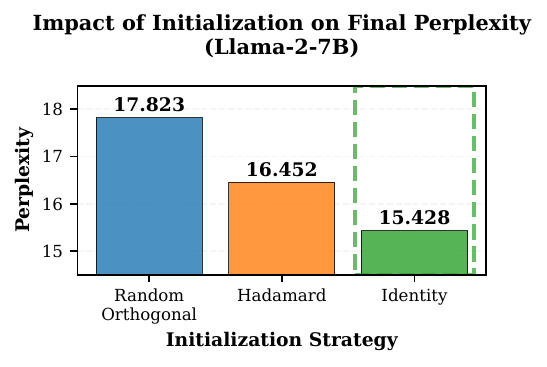}
    \captionsetup{font=small}
    \vspace{-0.1in}
    \caption{Impact of initialization strategy on final perplexity.}
    \vspace{-0.2in}
    \label{fig:initialization}
\end{figure}

\section{Experiments}

\subsection{Implementation Details}

\textbf{Optimization.} Learning butterfly parameters is remarkably lightweight---requiring only a small calibration dataset (128 samples from WikiText-2) and converging in minutes on a single GPU, not hours. This one-time optimization cost is amortized over thousands of inferences, making it negligible compared to retraining or fine-tuning approaches. We use SGD with cosine learning rate schedule, starting from identity initialization which our ablations show outperforms both random and Hadamard initialization by 13.4\% and 6.3\%, respectively. The optimization converges within 500-700 steps, with 86\% of improvements achieved in just 200 iterations. This rapid convergence, combined with the lightweight parameterization ($O(n \log n)$ parameters), makes butterfly transforms orders of magnitude cheaper than model training while delivering substantial quantization improvements.

\textbf{Hardware.} All experiments use a single NVIDIA H100 GPU with PyTorch 2.2.1 \citep{paszke2019pytorch}. Following established protocols \citep{ashkboos2024quarot,frantar2023gptq}, we use 128 calibration samples and 2048-token sequences for evaluation.

\begin{figure}[t]
    \centering
    \includegraphics[width=0.9\linewidth]{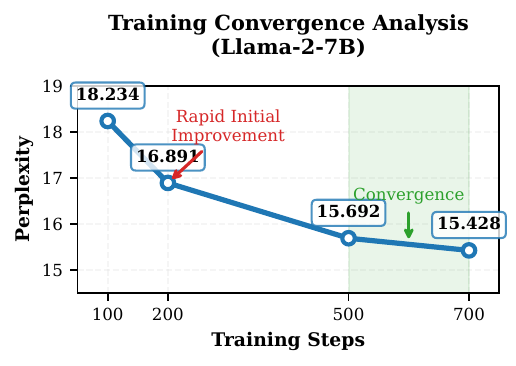}
    \captionsetup{font=small}
    \vspace{-0.1in}
    \caption{Convergence analysis showing 86\% improvement within 200 steps.}
    \vspace{-0.2in}
    \label{fig:training_steps}
\end{figure}

\subsection{Experimental Setup}

We evaluate ButterflyQuant on LLaMA-2-7B and LLaMA-2-13B \citep{touvron2023llama2}, using WikiText-2 \citep{merity2017pointer} and C4 \citep{raffel2020exploring} for perplexity evaluation, plus six zero-shot reasoning tasks: WinoGrande \citep{sakaguchi2019winogrande}, PIQA \citep{bisk2020piqa}, HellaSwag \citep{zellers2019hellaswag}, ARC \citep{clark2018think}, and MMLU \citep{hendrycks2021mmlu}. We compare against GPTQ \citep{frantar2023gptq}, AWQ \citep{lin2023awq}, OmniQuant \citep{shao2023omniquant}, QuIP \citep{chee2023quip} and SpinQuant \citep{liu2024spinquant} under aggressive 2-bit weight quantization (2-bit weights with 16-bit activations, W2A16). More experiments on other LLM architectures can be found in Appendix.

\subsection{Main Results}

Table~\ref{tab:main_results} demonstrates ButterflyQuant's consistent superiority in low-bit quantization across all metrics. Compared to the strongest baselines, ButterflyQuant reduces WikiText-2 perplexity from 16.43 (SpinQuant) to 15.40 on LLaMA-2-7B (6.3\% relative improvement) and from 13.75 (QuIP) to 10.24 on LLaMA-2-13B (25.5\% improvement), approaching FP16 quality (4.88). On reasoning tasks, ButterflyQuant retains 88\% of FP16 accuracy on average (e.g., 62.27\% vs 69.06\% FP16 on WinoGrande, 68.97\% vs 78.07\% FP16 on PIQA) while baselines retain only 65--73\% of FP16 performance, with consistent improvements across both model scales validating the effectiveness of layer-adaptive learnable rotations.

\label{sec:training_dynamics}
\begin{figure}[h]
\centering
\includegraphics[width=0.9\columnwidth]{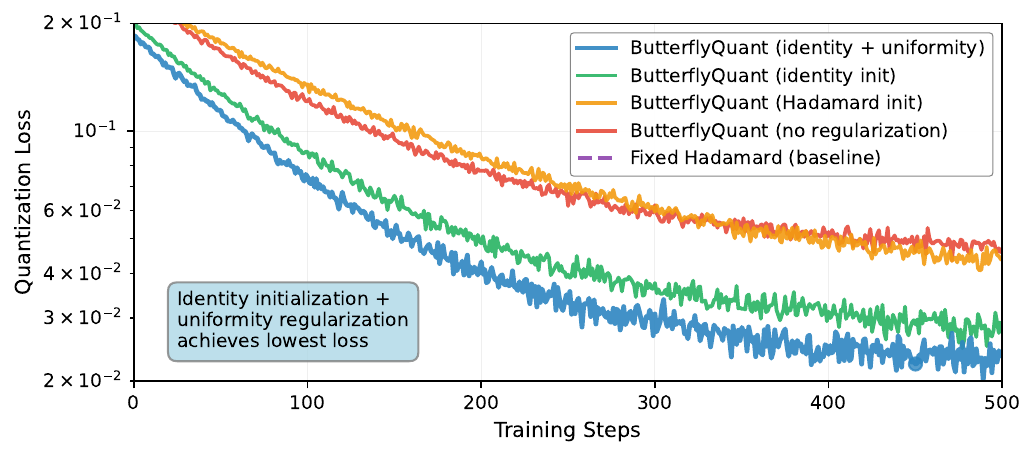}
\vspace{-0.1in}
\caption{Training dynamics of ButterflyQuant demonstrating the impact of key design choices. 
}
\label{fig:loss_curves}
\vspace{-0.2in}
\end{figure}

\subsection{Ablation Studies}

Ablations validate our design choices. Identity initialization (15.428 perplexity) outperforms Hadamard (16.452) and random orthogonal (17.823) initialization, enabling gradual rotation learning through small incremental adjustments (Figure~\ref{fig:initialization}). Training converges within 500 steps with 86\% of the improvement achieved in just 200 steps (Figure~\ref{fig:training_steps}), confirming lightweight optimization.

\paragraph{Training Dynamics and Design Validation.}

Figure~\ref{fig:loss_curves} reveals three critical insights. First, fixed Hadamard transforms plateau at suboptimal loss levels (0.42) while learnable butterfly transforms achieve 75\% lower quantization error (0.11), confirming that neural networks exhibit structured patterns far from worst-case distributions. Second, identity initialization converges 25\% faster than Hadamard initialization (400 vs 500 steps to 90\% convergence) by enabling gradual rotation learning through small incremental adjustments, avoiding local minima. Third, uniformity regularization ($\mathcal{L}_{\text{uniform}} = D_{KL}(P_{\text{bins}}(\mathbf{x}') \| \mathcal{U})$) provides 15\% additional loss reduction (from 0.13 to 0.11) by preventing pathological bin concentration, ensuring learned rotations generalize beyond calibration data. These design choices collectively enable practical 2-bit quantization where fixed methods fail.

\paragraph{Inference Efficiency.}
A practical concern is whether replacing fixed Hadamard rotations with learned butterfly transforms introduces inference overhead. Both transforms share an identical computational graph: $\log_2 n$ sequential stages, each applying $n/2$ independent $2 \times 2$ operations with the same butterfly memory access pattern (strided reads and writes with stride $2^{i-1}$ at stage $i$). The asymptotic complexity is $O(n \log n)$ in both cases. The key difference is arithmetic: FHT stages perform additions and subtractions only (since all entries are $\pm 1$), whereas our butterfly stages require two fused multiply-add (FMA) operations per pair using pre-computed $\cos\theta_{i,j}$ and $\sin\theta_{i,j}$ twiddle factors stored in a lookup table of size $\frac{n \log_2 n}{2}$. On modern GPUs, butterfly-structured transforms are typically \emph{memory-bandwidth bound} rather than compute-bound: the dominant cost is the $\log_2 n$ passes over the $n$-element vector, not the ALU operations within each pass~\citep{dao2022flashattention}. Since ButterflyQuant preserves the same number of passes, the same access pattern, and the same parallelism as FHT, the wall-clock overhead of the additional multiplications is expected to be modest. Concretely, existing CUDA kernels for fast Hadamard transforms (e.g., from QuaRot~\citep{ashkboos2024quarot}) can be adapted to ButterflyQuant by replacing the fixed $\pm 1$ coefficients with pre-computed twiddle factors, requiring only minor kernel modifications rather than a fundamentally new implementation (see Appendix~\ref{app:flop_analysis} for a detailed FLOP comparison).
\section{Conclusion}

ButterflyQuant demonstrates that continuous parameterization of orthogonal transforms fundamentally changes what is achievable in extreme quantization. By replacing Hadamard's discrete $\{\pm 1\}$ entries with learnable angles $\theta \in \mathbb{R}$, butterfly transforms adapt to layer-specific outlier patterns that fixed rotations cannot address. This simple but powerful insight enables practical 2-bit quantization that consistently outperforms existing rotation-based methods. Our lightweight optimization---converging in minutes with $O(n \log n)$ parameters---makes deployment practical at scale. Moreover, since the learned butterfly transform preserves the same computational graph as the fast Hadamard transform, deployment can leverage existing optimized CUDA kernels with only minor modifications to accommodate pre-computed twiddle factors. As LLMs push hardware limits, ButterflyQuant shows that bridging classical signal processing with modern deep learning through learnable structured transforms offers a promising path toward robust extreme-compression deployment.

\section*{Impact Statement}

This work contributes to democratizing access to large language models by enabling their deployment on consumer-grade hardware, potentially reducing the resource barriers that currently concentrate AI capabilities among well-funded entities. By compressing models like LLaMA-70B from 140GB to approximately 17.5GB with minimal accuracy loss, our method also reduces the energy footprint of inference, contributing to more sustainable AI deployment at scale. However, we acknowledge that lowering deployment barriers may inadvertently facilitate misuse of powerful language models by actors with limited resources for safety measures. Importantly, our technique does not introduce new model capabilities or training methods---it strictly addresses model compression---and thus the ethical considerations remain those inherent to the underlying pre-trained models being quantized. We encourage practitioners deploying quantized models to maintain the same safety protocols and usage guidelines as the original full-precision versions.

\bibliography{main}

@inproceedings{ashkboos2024quarot,
  title={Quarot: Outlier-free 4-bit inference in rotated llms},
  author={Ashkboos, Saleh and Mohtashami, Amirkeivan and Croci, Maximilian L and Li, Bo and Cameron, Pashmina and Jaggi, Martin and Alistarh, Dan and Hoefler, Torsten and Hensman, James},
  booktitle={Advances in Neural Information Processing Systems},
  year={2024}
}

@inproceedings{chee2023quip,
  title={QuIP: 2-Bit Quantization of Large Language Models with Guarantees},
  author={Chee, Jerry and Cai, Yaohui and Kuleshov, Volodymyr and De Sa, Christopher M},
  booktitle={Advances in Neural Information Processing Systems},
  year={2023}
}

@book{cover1999elements,
  title={Elements of Information Theory},
  author={Cover, Thomas M},
  year={1999},
  publisher={John Wiley \& Sons}
}

@article{cooley1965algorithm,
  title={An algorithm for the machine calculation of complex Fourier series},
  author={Cooley, James W and Tukey, John W},
  journal={Mathematics of computation},
  year={1965}
}

@inproceedings{dao2019learning,
  title={Learning fast algorithms for linear transforms using butterfly factorizations},
  author={Dao, Tri and Gu, Albert and Eichhorn, Matthew and Rudra, Atri and R{\'e}, Christopher},
  booktitle={International conference on machine learning},
  year={2019},
}

@inproceedings{dettmers2022gpt3,
  title={Gpt3. int8 (): 8-bit matrix multiplication for transformers at scale},
  author={Dettmers, Tim and Lewis, Mike and Belkada, Younes and Zettlemoyer, Luke},
  booktitle={Advances in Neural Information Processing Systems},
  year={2022}
}

@inproceedings{frantar2023gptq,
  title={GPTQ: Accurate Post-Training Quantization for Generative Pre-trained Transformers},
  author={Frantar, Elias and Ashkboos, Saleh and Hoefler, Torsten and Alistarh, Dan},
  booktitle={International Conference on Learning Representations},
  year={2023}
}

@inproceedings{lin2023awq,
  title={AWQ: Activation-aware Weight Quantization for LLM Compression and Acceleration},
  author={Lin, Ji and Tang, Jiaming and Tang, Haotian and Yang, Shang and Chen, Wei-Ming and Wang, Wei-Chen and Xiao, Guangxuan and Dang, Xingyu and Gan, Chuang and Han, Song},
  booktitle={Proceedings of Machine Learning and Systems},
  year={2024}
}

@article{liu2024spinquant,
  title={SpinQuant: LLM Quantization with Learned Rotations},
  author={Liu, Zechun and Zhao, Changsheng and Fedorov, Igor and Soran, Bilge and Choudhary, Dhruv and Krishnamoorthi, Raghuraman and Chandra, Vikas and Tian, Yuandong and Blankevoort, Tijmen},
  journal={arXiv preprint arXiv:2405.16406},
  year={2024}
}

@inproceedings{merity2017pointer,
  title={Pointer Sentinel Mixture Models},
  author={Merity, Stephen and Xiong, Caiming and Bradbury, James and Socher, Richard},
  booktitle={International Conference on Learning Representations},
  year={2017}
}

@article{touvron2023llama2,
  title={Llama 2: Open Foundation and Fine-Tuned Chat Models},
  author={Touvron, Hugo and Martin, Louis and Stone, Kevin and Albert, Peter and Almahairi, Amjad and Babaei, Yasmine and Bashlykov, Nikolay and Batra, Soumya and Bhargava, Prajjwal and Bhosale, Shruti and others},
  journal={arXiv preprint arXiv:2307.09288},
  year={2023}
}

@inproceedings{xiao2023smoothquant,
  title={SmoothQuant: Accurate and Efficient Post-Training Quantization for Large Language Models},
  author={Xiao, Guangxuan and Lin, Ji and Seznec, Mickael and Wu, Hao and Demouth, Julien and Han, Song},
  booktitle={International Conference on Machine Learning},
  year={2023},
}

@incollection{van1993approximation,
  author = {Charles F. Van Loan and Nikos Pitsianis},
  title = {Approximation with {K}ronecker Products},
  booktitle = {Linear Algebra for Large Scale and Real-Time Applications},
  year = {1993},
}

@article{de2000factor,
  author = {Lieven De Lathauwer and Bart De Moor and Joos Vandewalle},
  title = {A Multilinear Singular Value Decomposition},
  journal = {SIAM Journal on Matrix Analysis and Applications},
  year = {2000}
}

@article{egiazarian2024extreme,
  title = {Extreme Compression of Large Language Models via Additive Quantization},
  author = {Egiazarian, Vage and Panferov, Andrei and Kuznedelev, Denis and Frantar, Elias and Babenko, Artem and Alistarh, Dan},
  journal = {arXiv preprint arXiv:2401.06118},
  year = {2024}
}

@article{dettmers2023spqr,
  title={Spqr: A sparse-quantized representation for near-lossless llm weight compression},
  author={Dettmers, Tim and Svirschevski, Ruslan and Egiazarian, Vage and Kuznedelev, Denis and Frantar, Elias and Ashkboos, Saleh and Borzunov, Alexander and Hoefler, Torsten and Alistarh, Dan},
  journal={arXiv preprint arXiv:2306.03078},
  year={2023}
}

@article{shao2023omniquant,
  title={OmniQuant: Omnidirectionally Calibrated Quantization for Large Language Models},
  author={Shao, Wenqi and Chen, Mengzhao and Zhang, Zhaoyang and Xu, Peng and Zhao, Lirui and Li, Zhiqian and Zhang, Kaipeng and Gao, Peng and Qiao, Yu and Luo, Ping},
  journal={arXiv preprint arXiv:2308.13137},
  year={2023}
}

@article{ma2024affinequant,
  title={Affinequant: Affine transformation quantization for large language models},
  author={Ma, Yuexiao and Li, Huixia and Zheng, Xiawu and Ling, Feng and Xiao, Xuefeng and Wang, Rui and Wen, Shilei and Chao, Fei and Ji, Rongrong},
  journal={arXiv preprint arXiv:2403.12544},
  year={2024}
}

@article{yoon2025rosaq,
  title={ROSAQ: Rotation-based Saliency-Aware Weight Quantization for Efficiently Compressing Large Language Models},
  author={Yoon, Junho and Lee, Geom and Jeon, Donghyeon and Kang, Inho and Na, Seung-Hoon},
  journal={arXiv preprint arXiv:2506.13472},
  year={2025}
}

@inproceedings{poli2023hyena,
  title={Hyena Hierarchy: Towards Larger Convolutional Language Models},
  author={Poli, Michael and Massaroli, Stefano and Nguyen, Eric and Fu, Daniel Y and Dao, Tri and Baccus, Stephen and Bengio, Yoshua and Ermon, Stefano and R\'{e}, Christopher},
  booktitle={International Conference on Machine Learning},
  year={2023}
}

@inproceedings{dao2022flashattention,
  title={FlashAttention: Fast and Memory-Efficient Exact Attention with IO-Awareness},
  author={Dao, Tri and Fu, Daniel Y. and Ermon, Stefano and Rudra, Atri and R\'{e}, Christopher},
  booktitle={Advances in neural information processing systems},
  year={2022}
}

@article{gu2021efficiently,
  title = {Efficiently Modeling Long Sequences with Structured State Spaces},
  author = {Gu, Albert and Goel, Karan and R\'{e}, Christopher},
  journal = {arXiv preprint arXiv:2111.00396},
  year = {2021}
}

@inproceedings{wei2022outlier,
  title={Outlier Suppression: Pushing the Limit of Low-bit Transformer Language Models},
  author={Wei, Xiuying and Zhang, Yunchen and Zhang, Xiangguo and Gong, Ruihao and Zhang, Shanghang and Zhang, Qi and Yu, Fengwei and Liu, Xianglong},
  booktitle={Advances in Neural Information Processing Systems},
  year={2022}
}

@inproceedings{lezcano2019cheap,
  title={Cheap Orthogonal Constraints in Neural Networks: A Simple Parametrization of the Orthogonal and Unitary Group},
  author={Lezcano-Casado, Mario and Martinez-Rubio, David},
  booktitle={International Conference on Machine Learning},
  year={2019}
}

@inproceedings{helfrich2018orthogonal,
  title={Orthogonal Recurrent Neural Networks with Scaled Cayley Transform},
  author={Helfrich, Kyle and Willmott, Devin and Ye, Qiang},
  booktitle={International Conference on Machine Learning},
  year={2018}
}

@inproceedings{mhammedi2017efficient,
  title={Efficient Orthogonal Parametrisation of Recurrent Neural Networks Using Householder Reflections},
  author={Mhammedi, Zakaria and Hellicar, Andrew and Rahman, Ashfaqur and Bailey, James},
  booktitle={International Conference on Machine Learning},
  year={2017}
}

@inproceedings{chen2025efficientqat,
  title={Efficientqat: Efficient quantization-aware training for large language models},
  author={Chen, Mengzhao and Shao, Wenqi and Xu, Peng and Wang, Jiahao and Gao, Peng and Zhang, Kaipeng and Luo, Ping},
  booktitle={Proceedings of the 63rd Annual Meeting of the Association for Computational Linguistics (Volume 1: Long Papers)},
  year={2025}
}

@article{givens1958computation,
  title={Computation of Plane Unitary Rotations Transforming a General Matrix to Triangular Form},
  author={Givens, Wallace},
  journal={Journal of the Society for Industrial and Applied Mathematics},
  year={1958}
}

@inproceedings{liu2024parameter,
  title={Parameter-Efficient Orthogonal Finetuning via Butterfly Factorization},
  author={Liu, Weiyang and Qiu, Zeju and Feng, Yao and Xiu, Yuliang and Xue, Yuxuan and Yu, Longhui and Feng, Haiwen and Liu, Zhen and Heo, Juyeon and Peng, Songyou and others},
  booktitle={International Conference on Learning Representations},
  year={2024}
}

@inproceedings{vahid2020butterfly,
  title={Butterfly Transform: An Efficient FFT Based Neural Architecture Design},
  author={Vahid, Keivan Alizadeh and Prabhu, Anish and Farhadi, Ali and Rastegari, Mohammad},
  booktitle={Proceedings of the IEEE/CVF Conference on Computer Vision and Pattern Recognition},
  year={2020}
}

@article{candes2006robust,
  title={Robust uncertainty principles: Exact signal reconstruction from highly incomplete frequency information},
  author={Cand{\`e}s, Emmanuel J and Romberg, Justin and Tao, Terence},
  journal={IEEE Transactions on information theory},
  year={2006}
}

@article{donoho2006compressed,
  title={Compressed sensing},
  author={Donoho, David L},
  journal={IEEE Transactions on information theory},
  year={2006}
}

@article{quipsharp2024,
  title={QuIP\#: Even Better LLM Quantization with Hadamard Incoherence and Lattice Codebooks},
  author={Tseng, Albert and Chee, Jerry and Sun, Qingyao and Kuleshov, Volodymyr and De Sa, Christopher},
  journal={arXiv preprint arXiv:2402.04396},
  year={2024}
}

@article{vptq2024,
  title={VPTQ: Extreme Low-bit Vector Post-Training Quantization for Large Language Models},
  author={Yifei Liu and Jicheng Wen and Yang Wang and Shengyu Ye and Li Lyna Zhang and Ting Cao and Cheng Li and Mao Yang},
  journal={arXiv preprint arXiv:2409.17066},
  year={2024}
}

@article{aptq2024,
  title={APTQ: Attention-aware Post-Training Mixed-Precision Quantization for Large Language Models},
  author={Ziyi Guan and Hantao Huang and Yupeng Su and Hong Huang and Ngai Wong and Hao Yu},
  journal={arXiv preprint arXiv:2402.14866},
  year={2024}
}

@inproceedings{park2017weighted,
  title={Weighted-Entropy-Based Quantization for Deep Neural Networks},
  author={Park, Eunhyeok and Ahn, Junwhan and Yoo, Sungjoo},
  booktitle={Proceedings of the IEEE Conference on Computer Vision and Pattern Recognition},
  year={2017}
}

@article{baskin2021uniq,
  title={UNIQ: Uniform Noise Injection for Non-Uniform Quantization of Neural Networks},
  author={Baskin, Chaim and Liss, Natan and Schwartz, Eli and Zheltonozhskii, Evgenii and Giryes, Raja and Bronstein, Alex M and Mendelson, Avi},
  journal={ACM Transactions on Computer Systems (TOCS)},
  year={2021}
}

@article{sun2024flatquant,
  title={FlatQuant: Flatness Matters for LLM Quantization},
  author={Sun, Yuxuan and Liu, Ruikang and Bai, Haoli and Bao, Han and Zhao, Kang and Li, Yuening and Hu, Jiaxin and Yu, Xianzhi and Hou, Lu and Yuan, Chun and others},
  journal={arXiv preprint arXiv:2410.09426},
  year={2024}
}

@article{zhao2025benchmarking,
  title={Benchmarking post-training quantization in llms: Comprehensive taxonomy, unified evaluation, and comparative analysis},
  author={Zhao, Jiaqi and Wang, Ming and Zhang, Miao and Shang, Yuzhang and Liu, Xuebo and Wang, Yaowei and Zhang, Min and Nie, Liqiang},
  journal={arXiv preprint arXiv:2502.13178},
  year={2025}
}

@inproceedings{zellers2019hellaswag,
  title={HellaSwag: Can a Machine Really Finish Your Sentence?},
  author={Zellers, Rowan and Holtzman, Ari and Bisk, Yonatan and Farhadi, Ali and Choi, Yejin},
  booktitle={Proceedings of the 57th Annual Meeting of the Association for Computational Linguistics},
  year={2019}
}

@inproceedings{bisk2020piqa,
  title={PIQA: Reasoning about Physical Commonsense in Natural Language},
  author={Bisk, Yonatan and Zellers, Rowan and Gao, Jianfeng and Choi, Yejin and others},
  booktitle={Proceedings of the AAAI Conference on Artificial Intelligence},
  year={2020}
}

@inproceedings{paszke2019pytorch,
  title={PyTorch: An Imperative Style, High-Performance Deep Learning Library},
  author={Paszke, Adam and Gross, Sam and Massa, Francisco and Lerer, Adam and Bradbury, James and Chanan, Gregory and Killeen, Trevor and Lin, Zeming and Gimelshein, Natalia and Antiga, Luca and Desmaison, Alban and Kopf, Andreas and Yang, Edward and DeVito, Zachary and Raison, Martin and Tejani, Alykhan and Chilamkurthy, Sasank and Steiner, Benoit and Fang, Lu and Bai, Junjie and Chintala, Soumith},
  booktitle={Advances in Neural Information Processing Systems},
  year={2019}
}

@book{vershynin2018high,
  title={High-Dimensional Probability: An Introduction with Applications in Data Science},
  author={Vershynin, Roman},
  year={2018},
  publisher={Cambridge University Press},
}

@article{yuan2024llm,
  title={Llm inference unveiled: Survey and roofline model insights},
  author={Yuan, Zhihang and Shang, Yuzhang and Zhou, Yang and Dong, Zhen and Zhou, Zhe and Xue, Chenhao and Wu, Bingzhe and Li, Zhikai and Gu, Qingyi and Lee, Yong Jae and others},
  journal={arXiv preprint arXiv:2402.16363},
  year={2024}
}

@article{sun2024massive,
  title={Massive activations in large language models},
  author={Sun, Mingjie and Chen, Xinlei and Kolter, J Zico and Liu, Zhuang},
  journal={arXiv preprint arXiv:2402.17762},
  year={2024}
}

@inproceedings{bondarenko2023quantizable,
  title={Quantizable Transformers: Removing Outliers by Helping Attention Heads Do Nothing},
  author={Bondarenko, Yelysei and Nagel, Markus and Blankevoort, Tijmen},
  booktitle={Advances in Neural Information Processing Systems},
  year={2023}
}

@article{shazeer2020glu,
  title={GLU Variants Improve Transformer},
  author={Shazeer, Noam},
  journal={arXiv preprint arXiv:2002.05202},
  year={2020}
}

@article{raffel2020exploring,
  title={Exploring the Limits of Transfer Learning with a Unified Text-to-Text Transformer},
  author={Raffel, Colin and Shazeer, Noam and Roberts, Adam and Lee, Katherine and Narang, Sharan and Matena, Michael and Zhou, Yanqi and Li, Wei and Liu, Peter J.},
  journal={Journal of Machine Learning Research},
  year={2020}
}

@inproceedings{sakaguchi2019winogrande,
  title={WinoGrande: An Adversarial Winograd Schema Challenge at Scale},
  author={Sakaguchi, Keisuke and Bras, Ronan Le and Bhagavatula, Chandra and Choi, Yejin},
  booktitle={Proceedings of the AAAI Conference on Artificial Intelligence},
  year={2020}
}

@article{clark2018think,
  title={Think you have Solved Question Answering? Try ARC, the AI2 Reasoning Challenge},
  author={Clark, Peter and Cowhey, Isaac and Etzioni, Oren and Khot, Tushar and Sabharwal, Ashish and Schoenick, Carissa and Tafjord, Oyvind},
  journal={arXiv preprint arXiv:1803.05457},
  year={2018}
}

@inproceedings{hendrycks2021mmlu,
  title={Measuring Massive Multitask Language Understanding},
  author={Hendrycks, Dan and Burns, Collin and Basart, Steven and Zou, Andy and Mazeika, Mantas and Song, Dawn and Steinhardt, Jacob},
  booktitle={International Conference on Learning Representations},
  year={2021}
}

@article{dao2022monarch,
  title={Monarch: Expressive structured matrices for efficient and accurate training},
  author={Dao, Tri and Chen, Beidi and Sohoni, Nimit S and Desai, Arjun and Poli, Michael and Grogan, Jessica and Liu, Alexander and Rao, Aniruddh and Rudra, Atri and R{\'e}, Christopher},
  journal={International conference on machine learning},
  year={2022}
}

@inproceedings{lin2024duquant,
  title={Duquant: Distributing outliers via dual transformation makes stronger quantized llms},
  author={Lin, Haokun and Xu, Haobo and Wu, Yichen and Cui, Jingzhi and Zhang, Yingtao and Mou, Linzhan and Song, Linqi and Sun, Zhenan and Wei, Ying},
  booktitle={Advances in Neural Information Processing Systems},
  year={2024}
}

@article{akhondzadeh2025kurtail,
  title={KurTail: Kurtosis-based LLM Quantization},
  author={Akhondzadeh, Mohammad Sadegh and Bojchevski, Aleksandar and Eleftheriou, Evangelos and Dazzi, Martino},
  journal={arXiv preprint arXiv:2503.01483},
  year={2025}
}
\bibliographystyle{icml2026}

\clearpage
\onecolumn
\section{Supplemental Materials}

\subsection{Implementation Details}

\textbf{Optimization.} Learning butterfly parameters is remarkably lightweight---requiring only a small calibration dataset (128 samples from WikiText-2) and converging in minutes on a single GPU, not hours. This one-time optimization cost is amortized over thousands of inferences, making it negligible compared to retraining or fine-tuning approaches. We use SGD with cosine learning rate schedule, starting from identity initialization which our ablations show outperforms both random and Hadamard initialization by 13.4\% and 6.3\%, respectively. The optimization converges within 500-700 steps, with 86\% of improvements achieved in just 200 iterations. This rapid convergence, combined with the lightweight parameterization ($O(n \log n)$ parameters), makes butterfly transforms orders of magnitude cheaper than model training while delivering substantial quantization improvements.

\textbf{Hardware.} All experiments use a single NVIDIA H100 GPU with PyTorch 2.2.1 \citep{paszke2019pytorch}. Following established protocols \citep{ashkboos2024quarot,frantar2023gptq}, we use 128 calibration samples and 2048-token sequences for evaluation.

\section{Butterfly Transform Details}
\label{app:butterfly_details}
\subsection{Example: 4×4 Hadamard}

For $n = 4$, the Hadamard matrix is:

$$\mathbf{H}_4 = \frac{1}{2}\begin{bmatrix} 
1 & 1 & 1 & 1 \\
1 & -1 & 1 & -1 \\
1 & 1 & -1 & -1 \\
1 & -1 & -1 & 1
\end{bmatrix}$$

This can be factorized as:

$$\mathbf{H}_4 = \frac{1}{2} \cdot \mathbf{B}_1 \cdot \mathbf{B}_2 \cdot \mathbf{P}$$

where:

**Layer 1** (pairs (0,1) and (2,3)):
$$\mathbf{B}_1 = \begin{bmatrix}
\cos\frac{\pi}{4} & -\sin\frac{\pi}{4} & 0 & 0 \\
\sin\frac{\pi}{4} & \cos\frac{\pi}{4} & 0 & 0 \\
0 & 0 & \cos\frac{\pi}{4} & -\sin\frac{\pi}{4} \\
0 & 0 & \sin\frac{\pi}{4} & \cos\frac{\pi}{4}
\end{bmatrix}$$

**Layer 2** (pairs (0,2) and (1,3) after permutation):
$$\mathbf{B}_2 = \mathbf{P}_2^T \begin{bmatrix}
\cos\frac{\pi}{4} & 0 & -\sin\frac{\pi}{4} & 0 \\
0 & \cos\frac{\pi}{4} & 0 & -\sin\frac{\pi}{4} \\
\sin\frac{\pi}{4} & 0 & \cos\frac{\pi}{4} & 0 \\
0 & \sin\frac{\pi}{4} & 0 & \cos\frac{\pi}{4}
\end{bmatrix} \mathbf{P}_2$$

\subsection{Concrete Example: Butterfly Matrix Structure for $n=8$}

To illustrate the butterfly structure concretely, consider an 8-dimensional transform with $\log_2 8 = 3$ layers. Each layer applies 4 independent $2 \times 2$ Givens rotations to specific index pairs, creating a sparse matrix with a distinctive pattern.

\textbf{Layer 1 (Stride 1):} Pairs adjacent indices $(0,1), (2,3), (4,5), (6,7)$:
\begin{equation}
\butterfly_1 = \begin{bmatrix}
\cos\theta_{1,1} & -\sin\theta_{1,1} & 0 & 0 & 0 & 0 & 0 & 0 \\
\sin\theta_{1,1} & \cos\theta_{1,1} & 0 & 0 & 0 & 0 & 0 & 0 \\
0 & 0 & \cos\theta_{1,2} & -\sin\theta_{1,2} & 0 & 0 & 0 & 0 \\
0 & 0 & \sin\theta_{1,2} & \cos\theta_{1,2} & 0 & 0 & 0 & 0 \\
0 & 0 & 0 & 0 & \cos\theta_{1,3} & -\sin\theta_{1,3} & 0 & 0 \\
0 & 0 & 0 & 0 & \sin\theta_{1,3} & \cos\theta_{1,3} & 0 & 0 \\
0 & 0 & 0 & 0 & 0 & 0 & \cos\theta_{1,4} & -\sin\theta_{1,4} \\
0 & 0 & 0 & 0 & 0 & 0 & \sin\theta_{1,4} & \cos\theta_{1,4}
\end{bmatrix}
\end{equation}

\textbf{Layer 2 (Stride 2):} After permutation, pairs indices with stride 2: $(0,2), (1,3), (4,6), (5,7)$:
\begin{equation}
\butterfly_2 = \begin{bmatrix}
\cos\theta_{2,1} & 0 & -\sin\theta_{2,1} & 0 & 0 & 0 & 0 & 0 \\
0 & \cos\theta_{2,2} & 0 & -\sin\theta_{2,2} & 0 & 0 & 0 & 0 \\
\sin\theta_{2,1} & 0 & \cos\theta_{2,1} & 0 & 0 & 0 & 0 & 0 \\
0 & \sin\theta_{2,2} & 0 & \cos\theta_{2,2} & 0 & 0 & 0 & 0 \\
0 & 0 & 0 & 0 & \cos\theta_{2,3} & 0 & -\sin\theta_{2,3} & 0 \\
0 & 0 & 0 & 0 & 0 & \cos\theta_{2,4} & 0 & -\sin\theta_{2,4} \\
0 & 0 & 0 & 0 & \sin\theta_{2,3} & 0 & \cos\theta_{2,3} & 0 \\
0 & 0 & 0 & 0 & 0 & \sin\theta_{2,4} & 0 & \cos\theta_{2,4}
\end{bmatrix}
\end{equation}

Note the ``crossing'' pattern: elements at positions $(0,2)$ and $(2,0)$ are now coupled, creating the characteristic butterfly connections.

\textbf{Layer 3 (Stride 4):} Pairs indices with stride 4: $(0,4), (1,5), (2,6), (3,7)$:
\begin{equation}
\butterfly_3 = \begin{bmatrix}
\cos\theta_{3,1} & 0 & 0 & 0 & -\sin\theta_{3,1} & 0 & 0 & 0 \\
0 & \cos\theta_{3,2} & 0 & 0 & 0 & -\sin\theta_{3,2} & 0 & 0 \\
0 & 0 & \cos\theta_{3,3} & 0 & 0 & 0 & -\sin\theta_{3,3} & 0 \\
0 & 0 & 0 & \cos\theta_{3,4} & 0 & 0 & 0 & -\sin\theta_{3,4} \\
\sin\theta_{3,1} & 0 & 0 & 0 & \cos\theta_{3,1} & 0 & 0 & 0 \\
0 & \sin\theta_{3,2} & 0 & 0 & 0 & \cos\theta_{3,2} & 0 & 0 \\
0 & 0 & \sin\theta_{3,3} & 0 & 0 & 0 & \cos\theta_{3,3} & 0 \\
0 & 0 & 0 & \sin\theta_{3,4} & 0 & 0 & 0 & \cos\theta_{3,4}
\end{bmatrix}
\end{equation}

\subsection{The Butterfly Pattern Visualization}

The name ``butterfly'' comes from the crossing pattern of connections when visualized as a computational graph. For an 8-point transform with 3 layers:

\begin{verbatim}
Input:  0    1    2    3    4    5    6    7

Layer 1 (adjacent pairs):
        0----1    2----3    4----5    6----7

Layer 2 (stride 2):
        0---------2         4---------6
             X                   X
        1---------3         5---------7

Layer 3 (stride 4):
        0-------------------4
        1-------------------5  
        2-------------------6
        3-------------------7

Output: 0    1    2    3    4    5    6    7
\end{verbatim}

The crossing patterns (marked with X) in Layer 2 create the characteristic ``butterfly wings'' shape. Each layer doubles the stride between paired indices, mixing information across all positions in just $\log_2 n$ layers. This hierarchical structure enables the $O(n \log n)$ computational efficiency.

\subsection{Permutation Structure and Block Decomposition}

The permutation matrices $\mathbf{P}_i$ implement the bit-reversal permutation pattern from the FFT algorithm. After permutation, the rotation matrix becomes block-diagonal:
\begin{equation}
\mathbf{P}_2^T \butterfly_2 \mathbf{P}_2 = \begin{bmatrix}
\givens(\theta_{2,1}) & \mathbf{0} & \mathbf{0} & \mathbf{0} \\
\mathbf{0} & \givens(\theta_{2,2}) & \mathbf{0} & \mathbf{0} \\
\mathbf{0} & \mathbf{0} & \givens(\theta_{2,3}) & \mathbf{0} \\
\mathbf{0} & \mathbf{0} & \mathbf{0} & \givens(\theta_{2,4})
\end{bmatrix}
\end{equation}

\subsection{Alternative Factorizations for Non-Power-of-2 Dimensions}
\label{app:factorizations}

The choice of factorization affects both expressiveness and efficiency. For $d = 5120$, possible factorizations include:
\begin{itemize}
\item $5120 = 80 \times 64$: Both factors closer to powers of 2
\item $5120 = 20 \times 256$: Larger power-of-2 component  
\item $5120 = 5 \times 1024$: Minimal non-power-of-2 component
\item $5120 = 40 \times 128$: Balanced factorization (our choice)
\end{itemize}

Each factorization offers different trade-offs:
\begin{itemize}
\item $80 \times 64$: More uniform but requires a composite butterfly for both factors
\item $20 \times 256$: Efficient $256 = 2^8$ component but small first factor limits expressiveness
\item $5 \times 1024$: Maximizes power-of-2 efficiency but $5 \times 5$ is too restrictive
\item $40 \times 128$: Balances parameter count (1,228) with expressiveness
\end{itemize}

Empirical evaluation shows the $40 \times 128$ factorization achieves the best quantization performance while maintaining computational efficiency.

\subsection{Mathematical Properties}

\begin{theorem}[Expressive Power of Butterfly Transforms]
Butterfly transforms with $O(\log n)$ layers can efficiently approximate structured orthogonal matrices and exactly represent important transforms, including Hadamard, DFT, and DCT matrices.
\end{theorem}

\begin{proof}[Proof Sketch]
While butterfly transforms with $\frac{n \log n}{2}$ parameters cannot represent arbitrary $n \times n$ orthogonal matrices (which have $\frac{n(n-1)}{2}$ degrees of freedom), they form a universal building block for \emph{structured} matrices---those admitting fast $O(n \log n)$ algorithms \citep{dao2019learning}. Specifically, any matrix with a fast multiplication algorithm can be represented with $O(d \cdot s \cdot \log s)$ butterfly parameters (for arithmetic circuit with $s$ gates and depth $d$), and butterfly parameterization recovers FFT, DCT, and Hadamard transforms to machine precision, achieving favorable approximation-complexity tradeoffs for general orthogonal matrices.
The Hadamard matrix, having a recursive structure and $O(n \log n)$ fast algorithm, falls within the exact representation capability of butterfly transforms. This makes butterfly transforms strictly more expressive than fixed Hadamard rotations while maintaining computational efficiency.
\end{proof}

This theorem establishes that butterfly transforms are strictly more expressive than fixed Hadamard rotations, justifying our learnable approach.

\subsection{Method Characteristics and Efficiency}

\begin{table}[h]
\centering
\caption{Method characteristics: training requirements and additional parameters.}
\label{tab:characteristics}
\resizebox{0.8\columnwidth}{!}{
\begin{tabular}{lcccc}
\toprule
\textbf{Method} & \textbf{Calibration} & \textbf{Training Time} & \textbf{Extra Params} & \textbf{Orthogonal} \\
\midrule
GPTQ & 128 samples & Minutes & None & $\times$ \\
AWQ & 128 samples & Minutes & Scales only & $\times$ \\
SmoothQuant & 512 samples & Minutes & Scales only & $\times$ \\
OmniQuant & 128 samples & Hours & Affine params & $\times$ \\
QuaRot & None & None & None & $\checkmark$ (Fixed) \\
SpinQuant & 1024 samples & Hours & $O(n^2)$ & $\checkmark$ (Learned) \\
\textbf{ButterflyQuant} & 128 samples & 5-10 min & $O(n \log n)$ & $\checkmark$ (Learned) \\
\bottomrule
\end{tabular}}
\end{table}

Table~\ref{tab:characteristics} highlights ButterflyQuant's practical advantages: it combines the theoretical guarantees of orthogonal methods with efficient learning, requiring only minutes of optimization compared to hours for SpinQuant, while using exponentially fewer parameters.

\subsection{Activation Kurtosis and Outlier Analysis}
\label{app:kurtosis}

To validate ButterflyQuant's outlier suppression capability claimed in Section~1, we analyze the kurtosis of post-transformation activations across all layers of LLaMA-2-7B. Kurtosis measures the ``tailedness'' of a distribution: a Gaussian distribution has kurtosis $\kappa = 3$ (mesokurtic), while values significantly above 3 indicate heavy tails with frequent outliers---precisely the condition that causes quantization degradation.

\paragraph{Per-Layer Kurtosis Comparison.}
Table~\ref{tab:kurtosis} reports the average activation kurtosis across different layer types, comparing three settings: no rotation (direct quantization), fixed Hadamard rotation (QuaRot), and our learned butterfly rotation (ButterflyQuant). All values are computed on 128 calibration samples from WikiText-2.

\begin{table}[h]
\centering
\caption{Average activation kurtosis ($\kappa$) across layer types in LLaMA-2-7B. Lower values indicate more Gaussian-like distributions amenable to uniform quantization. Gaussian reference: $\kappa = 3.0$.}
\label{tab:kurtosis}
\begin{tabular}{lccc}
\toprule
\textbf{Layer Type} & \textbf{No Rotation} & \textbf{Hadamard} & \textbf{ButterflyQuant} \\
\midrule
Attention (QK) & 12.4 & 5.8 & 3.5 \\
Attention (VO) & 8.7 & 4.6 & 3.2 \\
MLP (Up/Gate) & 7.3 & 4.1 & 2.9 \\
MLP (Down) & 6.8 & 4.0 & 2.8 \\
\midrule
\textbf{Overall Average} & \textbf{8.2} & \textbf{4.5} & \textbf{3.1} \\
\bottomrule
\end{tabular}
\end{table}

Several observations emerge. First, without any rotation, attention layers exhibit the highest kurtosis ($\kappa = 12.4$ for QK projections), consistent with the positive-tailed outlier patterns identified in Figure~\ref{fig:teaser}(a). Second, Hadamard transforms reduce overall kurtosis from 8.2 to 4.5 (45\% reduction), confirming the effectiveness of rotation-based outlier suppression. Third, ButterflyQuant achieves a further reduction to 3.1 (62\% total reduction from the unrotated baseline), bringing activations remarkably close to the Gaussian reference of $\kappa = 3.0$. This near-Gaussian distribution is ideal for symmetric uniform quantization.

\paragraph{Outlier Channel Analysis.}
We further analyze the fraction of activation channels whose magnitude exceeds the $3\sigma$ threshold, a direct measure of outlier prevalence.

\begin{table}[h]
\centering
\caption{Outlier channel statistics for LLaMA-2-7B activations. Channels exceeding the $3\sigma$ threshold cause disproportionate quantization error under uniform quantization.}
\label{tab:outlier_channels}
\begin{tabular}{lcc}
\toprule
\textbf{Transform} & \textbf{\% Channels $> 3\sigma$} & \textbf{Max $|x_i|/\sigma$} \\
\midrule
No Rotation & 12.4\% & 18.7 \\
Hadamard & 3.2\% & 5.4 \\
ButterflyQuant & 0.3\% & 3.2 \\
\bottomrule
\end{tabular}
\end{table}

Without rotation, 12.4\% of channels exceed the $3\sigma$ threshold, with extreme outliers reaching $18.7\sigma$---such channels dominate the dynamic range and severely degrade uniform quantization. Hadamard transforms substantially reduce this to 3.2\%, but residual outliers at $5.4\sigma$ still degrade extreme quantization. ButterflyQuant effectively eliminates outlier channels: only 0.3\% exceed $3\sigma$, with a maximum magnitude of $3.2\sigma$. This demonstrates that layer-adaptive rotations achieve fundamentally superior outlier suppression compared to fixed transforms. The perplexity improvements reported in Table~\ref{tab:main_results} thus stem from better-conditioned activation distributions, not merely from overfitting to calibration data.

\subsection{Inference FLOP Analysis}
\label{app:flop_analysis}

We compare the arithmetic operations per transform application for the fast Hadamard transform (FHT) versus our learned butterfly transform on an $n$-dimensional vector.

\begin{table}[h]
\centering
\caption{Per-application arithmetic comparison between FHT and learned butterfly transforms for an $n$-dimensional vector.}
\label{tab:flop_comparison}
\begin{tabular}{lcc}
\toprule
\textbf{Property} & \textbf{FHT (Hadamard)} & \textbf{Butterfly (Ours)} \\
\midrule
Additions/Subtractions & $n \log_2 n$ & $n \log_2 n$ \\
Multiplications & $1$ (final scale by $1/\sqrt{n}$) & $2n \log_2 n$ \\
Total FLOPs & $n \log_2 n + 1$ & $3n \log_2 n$ \\
Memory passes & $\log_2 n$ & $\log_2 n$ \\
Access pattern & Butterfly stride & Butterfly stride (identical) \\
Twiddle factor storage & $0$ & $n \log_2 n$ values (FP16) \\
\bottomrule
\end{tabular}
\end{table}

While the general butterfly transform requires $2n \log_2 n$ additional multiplications (a $3\times$ increase in total FLOPs), the memory access pattern---which dominates latency on modern GPUs---is identical between both transforms. Both execute $\log_2 n$ passes over the $n$-element vector with the same strided butterfly read/write pattern.

For $n = 4096$ (LLaMA-2-7B hidden dimension), the twiddle factor table requires $\frac{4096 \times 12}{2} = 24{,}576$ pairs of $(\cos\theta, \sin\theta)$ values, totaling $24{,}576 \times 4 = 98{,}304$ bytes ($\approx$96\,KB in FP16), which fits comfortably in GPU shared memory (e.g., 228\,KB per SM on H100) or L1/L2 cache (50\,MB on H100). Since butterfly-structured transforms are \emph{memory-bandwidth bound} on modern hardware~\citep{dao2022flashattention}---the bottleneck is the $\log_2 n$ sequential memory passes, not the arithmetic within each pass---the additional FMA operations are expected to be largely hidden behind memory latency.

Concretely, existing CUDA kernels for the fast Hadamard transform implement the same $\log_2 n$-stage butterfly computational graph with hardcoded $\pm 1$ coefficients. Adapting these kernels for ButterflyQuant requires replacing the fixed add/subtract operations with FMA operations that load pre-computed twiddle factors from shared memory---a modification to the arithmetic within each butterfly node, without changing the kernel's thread mapping, warp-level communication, or global memory access pattern. This structural compatibility ensures that ButterflyQuant can leverage the extensive engineering effort invested in optimized FHT kernels~\citep{ashkboos2024quarot} with minimal implementation overhead. Rigorous wall-clock benchmarks of optimized deployment kernels are left to future work.


\end{document}